\documentclass[12pt, reqno]{amsart}
\usepackage{geometry}
\usepackage{amsthm}
\usepackage{amssymb}
\newtheorem{axiom}{Axiom}[section]
\newtheorem{definition}{Definition}[section]
\newtheorem{theorem}{Theorem}[section]

\usepackage{enumerate}
\usepackage{graphicx}
\graphicspath{{./images/}}
\usepackage[T1]{fontenc}
\usepackage[utf8]{inputenc}
\usepackage{hyperref}

\numberwithin{equation}{section}

\usepackage{xcolor}

\title[Why Nature runs is parallel?]{Parallel Computation is ESS}
\author{Nabarun Mondal}
\address{D.E.Shaw \& Co. India, Hyderabad }
\email{mondal@deshaw.com}
\thanks{Dedicated to my late professor Dr. Prashanta Kumar Nandi. \\ 
Dedicated to my parents without their presence we are nothing. \\
Big thanks to: Abhishek Chanda,Gurpreet Singh (Manager),Waseem Fazal,Sahiba
Khurana, Jaipal Reddy : You all have been constant support. \\
In Memory of : Dhrubajyoti Ghosh, my closest human being. RIP. 
}

\author{Partha P. Ghosh}
\address{Microsoft India, Hyderabad }
\email{parthag@microsoft.com}


\begin{document}


\begin{abstract}
There are enormous amount of examples of Computation in nature,
exemplified across multiple species in biology. 
One crucial aim for these computations across all life forms
their ability to learn and thereby increase the chance of their survival. 
In the current paper a formal definition of autonomous learning is proposed.
From that definition we establish a Turing Machine model for learning,
where rule tables can be added or deleted, but can not be modified. 
Sequential and parallel implementations of this model are discussed.
It is found that for general purpose learning based on this model, 
the implementations capable of parallel execution would be evolutionarily stable.
This is proposed to be of the reasons why in Nature parallelism in computation is found in abundance. 
\end{abstract}

\maketitle

\begin{section}{Introduction}\label{intro}
Computations abound in nature, and it is not hard to fathom that 
one of the purposes of the natural computation is to learn.
A more learned individual would thereby gather advantage over the others \cite{rd1},
and survive. How learning evolved in nature remains a crucial question to be answered.

The extended Church Turing Thesis \cite{rp1}\cite{rp2}   
\cite{cd}\cite{amu}\cite{ms} normally stated in the form:
\emph{any computation happening in nature must also have a Turing Machine analog;}
suggests that no computation in nature can exceed the power of a Turing Machine \cite{at1}.
This intriguing thought of simulating nature on Turing Machines 
provoked the study of machine learning and artificial intelligence \cite{at2} \cite{js}.

One school of thought about machine learning 
follows the sequential way of creating machines which are capable of learning.
In fact scholars from the sequential school argue \cite{rp1}\cite{rp2} that  
parallel systems won't have any added advantage over the sequential ones, 
because both are reducible to Turing Machines.

Another school of thought \cite{ab} questions the sequential learning strategies, 
as nature seems to be inherently running in parallel. 
The massive parallel structures of neurones in brain \cite{vsr} prompted
the artificial neural network studies \cite{rf} \cite{rj}, and it is  
now known that bacterial colony learn by pooling in 
their individual (bacterium) resources \cite{ebj} which makes the learning strategy 
parallel by definition. Scholars from the \emph{parallel} school ask the question :-

\emph{Given Nature is inherently parallel,
how useful the sequential way of learning going to be?}

It is well known that computationally both of these models actually yield the same power 
\cite{ms} \cite{jw} : \emph{every parallel system must have a sequential analog.} 
The sequential models would take more time, less computing power, and less space, 
while the parallel ones would take less time, more computing power and more space.
Clearly then, the answer to \emph{``why there are parallelism in nature?''} lies not in the computing power
of the abstract Turing Machine,  but had to do with how each model 
has evolved in nature with no supervision ( a.k.a blind designoid ) \cite{rd2}.

Can one then decisively argue the need for which 
parallel computation is prevailing in Nature? 
This question begs an answer because parallelism in execution 
is harder to attain in the artificial settings of man made computing machines.
Artificial parallel systems are faster, but too simple
when compared against the natural systems like human brain, 
which are slower but far more complex\cite{rp1}.

In the present paper we take up this exact problem.
We notice that the real problem about systems evolving in nature is:-

\emph{Given redesign is not possible in nature,
what type of computation circuitry would evolutionarily arise and would become dominant?}

We discuss the abstract learning procedure in the Section \ref{learning}.
We establish that there is an abstract autonomous (blind,directionless) 
system capable of learning.

Section \ref{models} discusses sequential and parallel design 
(Dawkins \emph{designoid}, as any evolved, blind organization 
is not a design in a true sense \cite{rd1}). 
In the Section \ref{comparison} 
we compare these two models from optimality or resources standpoint.
We find that for general purpose learning, parallelism has evolutionary advantage.
It is not only because it is faster, but also because 
no other \emph{computationally}  better organization 
is possible, and \emph{computationally} it can not be improved.
Finally in the Section \ref{evo-l} we put all
these results together and show that once parallel strategy 
(in the sense of game theory) has invaded the population,
it would become dominant, and stay dominant, 
because it is an \emph{evolutionarily stable strategy}.
Therefore, we conclude that, the theory of computation 
and game theory together can explain the prevalence of 
parallelism in computational circuitry in nature. 
\end{section}

\begin{section}{Learning}\label{learning}
Everyone has an intuitive idea about what learning is.
But that idea relies upon another intuitive idea of what is knowledge.
For example \emph{Learning is, knowing what one did not know earlier}.

In this sense, there are some properties of learning one can summaries:-

\begin{definition}\label{gen-learning}
\textbf{Informal Description of Learning.}

After a system has learned, at some time $t_L$ the following properties hold:-
\begin{enumerate}
\item{ The system could achieve something which was unachievable at $t$ , $t < t_L$.}
\item{ The system would continue achieving everything it could achieve before $t_L$.}
\item{ The system should be autonomous, devoid of supervision by any intelligent agent. }
\end{enumerate}
\end{definition}

The third point needs elaboration. Once the system is set into motion,
after that no tweaking should be done with the system.
In specificity \emph{getting out of the system to perform system optimization}
is not allowed at all for autonomous learning.
This jumping out of the system idea is elaborated heavily in \cite{dh}.
 
We are not into the philosophies of \emph{understanding}. 
Searle's Chinese room argument \cite{js} and related arguments both in favour or against 
elaborates the lack of formalism in \emph{understanding}.

In a formal sense then, without getting into the \emph{understanding} ,
if there is a set of \emph{knowledge}
associated with the system  which can
somehow be identified as the set $\mathcal{K}(t_1)$ at time $t_1$,
and at a later time $t_2$ with the set $\mathcal{K}(t_2)$, then the statement
\emph{a learning took place between $t_1$ to $t_2$} 
is equivalent to the formal statement:-
$$
\mathcal{K}(t_1)  \subset \mathcal{K}(t_2) . 
$$ 
The learned knowledge can be modelled by :-
$$
\mathcal{L}(t_1,t_2) = \mathcal{K}(t_2) \setminus \mathcal{K}(t_1)
$$

But these semi-formal definitions would not really formalize learning 
given that the set $\mathcal{K}(t)$ was not formally defined.
Can we formally quantify the set $\mathcal{K}(t)$ ? 

At this point, we argue that the existence of $\mathcal{K}(t)$
can not be measured without the effect of $\mathcal{K}(t)$
exhibited by the system, which should only be identified by experimentation.
If the \emph{Extended Church Turing thesis} \cite{rp1}\cite{rp2}\cite{ms} 
is true, that is any model of computation
has an analogous Turing Machine model, then,
the effect of $\mathcal{K}(t)$ can be found 
in pretty straight forward manner. Considering a\emph{system}
that can \emph{recognize certain input sets}, we can define  \emph{knowledge}
as the \emph{input set the system recognises}.

\begin{definition}\label{system}
\textbf{System.}

Let $\Sigma$ be an alphabet, and $\omega, \alpha,\rho \not \in \Sigma$.
Let a system $S$ be defined by a ``black box'' with input $I$ a string from $\Sigma$ :-
$$
I \in \Sigma^*
$$
and with output symbol $\mathcal{O}$:-
$$
\mathcal{O} \in \{ \omega , \alpha , \rho \}
$$
Taking input sets the systems output to $\mathcal{O} = \omega$, (`working' symbol). 
Then later the system can either accepts the input string 
by outputting $\alpha$, or reject it by outputting $\rho$, or can do neither.
The set of all strings the system ``accepts'' be called as $C(S)$, the language class
accepted by `$S$'.
Hence, 
$$
S : \Sigma^* \to \{ \omega , \alpha , \rho \} 
$$
\end{definition}

To give an example take a bacterial colony `$S$' which can produce glucose from a set of chemical soups $\Psi$.
The set of all chemical soups, from which the colony can produce glucose can be called the language $\Psi = C(S)$ 
while, the presence of glucose after some time can be taken as outputting the symbol $\alpha$.

\begin{definition}\label{formal-learning}
\textbf{Formal Definition of Learning.}

Let a system (definition \ref{system}) at time $t$ 
accepts the language class $C(S(t))$. 
The system has learned within time $t_1 \to t_2 , t_1 \le t_2 $ 
iff:- 
$$
C(S(t_1)) \subset  C(S(t_2)) 
$$
The knowledge of the system $\mathcal{K}(S) = C(S)$, and learning is precisely 
defined as:-
$$
\mathcal{L}(t_1,t_2) = C(S(t_2)) \setminus  C(S(t_1)) 
$$
\end{definition}   

Comparing the definitions \eqref{gen-learning} with 
\eqref{formal-learning} we can see that \eqref{formal-learning}
is just a type of \eqref{gen-learning} which matches with the intuition.
Knowledge accumulation is now experimentally verifiable. 
To show that a system has learned within the time interval $(t_1,t_2)$, 
one has to find a string  $w_2  \in C(S(t_2))$ such that $w_2 \not \in C(S(t_1))$. 
We also note that the strict subset inclusion $C(S(t_i)) \subset C(S(t_j))$ for $i < j$
makes learning a \emph{filtration} \cite{jd} \cite{klc} process. 
 
Comparing with the earlier example, given that the colony could produce glucose from one for 
at least one additional chemical soup, let's call it $\mathcal{K}$ that it was not able to use earlier, 
would indeed qualify as \emph{learning}.
 
But what is the \emph{natural} mechanism of this \emph{recognition or acceptance?} 
If strong Church Turing Thesis is correct, then the recognition 
would have an equivalent Turing Machine model. From that perspective, 
a system can be defined as follows:-
 
\begin{definition}\label{system-turing}
\textbf{System Modelled by Universal Turing Machine.}

A system $S_T$  modelled by an Universal Turing machine is a mechanism comprising of 
a set of rule-tables : $R(t)=\{ r_1, r_2, ... \}$, which can be used as programs for the universal Turing machine.
The set $R(t)$ changes over time, denoting the time evolution of the system as $S_T(t)$.
The universal Turing machine can nondeterministically choose a rule-table(or rule-book) to recognise a string. 
The language class $C(S_T(t))$ is entirely determined by $R(t)$.
Let the language class accepted by simulating rule table $r_k$ be $C(r_k)$.
Then, 
$$
C(S_T(t)) = \bigcup C(r_k)
$$  
\end{definition}   

Now, this modelling has interesting consequences. 
The time evolution of $R(t)$ has to be such that the language class $C(R(t))$
stays a filtration to suggest that the system `$S_T$' is \emph{learning}.
What would be the possible mechanism of the time evolution of $R(t)$?

The only way to change the language class $C(S)$ is to change the set of the rule-books.
But what kind of change would be allowed? The rule books $r_k$ are programs, 
therefore, random change won't always work, and then, there is no warranty that the result would remain
a filtration after a change, that is the relation $C(r_k(t_1)) \subset C(r_k(t_2))$ will hold.

However, there is another way, which will always have the filtration condition satisfied. 
That is introducing new rule-tables in the set $R$, instead of modifying any of them.
This solves the problem of \emph{How a modification of rule-book always work}?

How a new rule table initially would get \emph{created} is a different problem, and is not being considered here.
But given that a new rule table somehow gets inserted in `$R$' at $t_2$, then, we must have:-
$$
C(R(t_1)) \subseteq C(R(t_2)) \; ; \; R(t_2) = R(t_1) \cup \{ r_{new} \} 
$$  
with equality holding only when  $C(r_{new}) \setminus C(S(t_1)) = \varnothing$.

This becomes a crucial hypothesis: \emph { new rule-book (programs) gets created and added to the existing set,
instead of modification of the existing rule-books which we hypothesize to be rare}.

These ideas are formalized in the next set of axioms.
\begin{axiom}\label{n-l}
\textbf{Natural Learning Systems Axioms.}

\begin{enumerate}
\item{ Learning is as defined as definition \ref{formal-learning} and is a filtration process.}
\item{ The only permissible operation to modify $R$ is the addition of a new rule-book $r_{new}$ to 
the existing rules set $R$ , where $\mathcal{R}$ is the set of all rule-books :- 
$$
\mathcal{A} : \mathcal{R} \to \mathcal{R}  
$$
such that :-
$$
\mathcal{A}(R(t_n)) = R( t_{n+1} ) = R(t_n) \cup \{ r_{new} \}
$$
}

\end{enumerate}
\end{axiom}

We now show that a model exists which is capable of autonomous learning as in 
definition \ref{formal-learning}. It adhere to the axioms of \eqref{n-l} 
and learns without any help from any supervisor.
Simply speaking we show that a designoid \cite{rd1} can learn.
Again, here we wont be discussing understanding. Learning  
here is a synonym for \emph{recognizing strings(information content) which was not recognized earlier}.

\begin{theorem}\label{model-learning}
\textbf{There is a learning process following Axiom \eqref{n-l}.}

Let $S(t_n)$ be a system at time $t_n$.
Let a new rule-book be added using operation $\mathcal{A}$, at time $t_n$ such that:-
\begin{equation}\label{cond-learning}
C(r_{new}) \setminus C(S(t_n)) \ne \varnothing 
\end{equation}
then, the operation $\mathcal{A}$ induced learning on the underlying system `$S$'.
\end{theorem}

\begin{proof}[Proof of the theorem \ref{model-learning} ]
Obvious from the discussion of the previous paragraphs.
Given this condition is met, the operation $\mathcal{A}$ will be inducing learning
to a system by definition. 
\end{proof}

So how  $r_{new}$ gets created? Anything random might work,
but can take many number of steps, with lucky (in a statistical sense) breaks.
In this paper we assume that some process induces creation of the newer rule book $r_{new}$,
but the mechanism is not what we are interested in.
We should note that not always adding a rule-set would induce learning, 
that is why the condition \eqref{cond-learning} is crucial. 

The point to be noted here is that the system which would start \emph{learning},
and the system after some time evolution, can be \emph{fundamentally} different, due to the 
nature of the rule-book being used. To call $S(t_{old})$ and $S(t_{new})$ different system
is analogous to the \emph{Ship of Theseus} or \emph{Theseus Paradox}.
It raises the question of whether an object which had some or all its components replaced 
remains fundamentally the same object, albeit in our case, in some sense $S(t_{old}) \subseteq S(t_{new})$,
due to filtration nature of learning.

\end{section}

\begin{section}{Practical Models Of The Learner}\label{models}
We have established that there exists a system depicted in Theorem \ref{model-learning}
which is capable of learning. 
The systems which are capable of learning 
subsequently would be called \emph{learners}. 

It is easy to notice that by the Theorem \ref{model-learning}, the way a learner
really learns is adding a new rule table to its repository of existing rule tables.
So, at every steps of learning, formally one new rule table gets added, and the learner
conveniently must switch from one table to another to accept a string. \emph{Note that 
no internal shuffling of the rules are allowed, the rule tables are atomic building blocks
by axiom \eqref{n-l}}.
With this idea in mind we can now formally define a physically possible learner as follows:-

\begin{definition}\label{learner}
\textbf{A Learning System : Learner.}

A learner is a system $S$ comprise of a set of rule tables $R=\{r_k\}$,
which can be simulated by Universal Turing Machine(s).
By simulating a rule table $r_k$ it can accept strings $w_k \in C_k \subset C(r_k)$.
It is capable of adding new rule table to $R$ 
via a process called learning. 
To learn the system can add a new rule-book $r_n$ to $R$:-
$$
R(t_2) = R(t_1) \cup \{ r_n \} 
$$
The language class accepted by a learner is bound by the relation :-
$$
C(S) \subseteq \bigcup_{k} C(r_k)
$$ 
\end{definition}

\begin{definition}\label{complete-system}
\textbf{Complete Learner.}

Let's assume a learner $S$ has a set of rule tables as $R= \{ r_k \}$.
Let the individual language class for each $r_k$ be $C(r_k)$.
Let the language class accepted by the learner be $C_L$.
The learner $S$ is said to be complete iff:-
$$
C_L = \bigcup_{k} C(r_k)
$$
\end{definition}

As for mechanism of implementation of a learner, 
it can belong to two general classes \emph{tandem} or \emph{sequential}
class, and \emph{parallel} class.

\begin{definition}\label{tandem}
\textbf{ Sequential Learner. }

To accept a string a sequential learner sequentially picks one rule table $r_k$ from 
$R$ and simulates it using the Universal Turing Machine,
until either no unused rule table exists, or there is an accept. 
\end{definition}

Which rule table to be used after which rule table becomes of importance now.
There is no obvious answer to that. In fact this question, 
the inherent sequential nature limits the language class the sequential system can accept.
This is the subject matter of the next theorem.

\begin{theorem}\label{non-union}
\textbf{Language Class Accepted by the Sequential Learner.}

Sequential learner accepts a language class $C_S$ which is generally 
a strict subset of the union of the language class of all the rule tables $r_k$s.
$$
C_S \subset \bigcup_{k} C(r_k)
$$
\end{theorem}
\begin{proof}
We construct the exact language class accepted by 
the sequential learner.
We begin by constructing the class of strings which 
would be accepted by the learner. 
Let's define $h(s,r_k)$ as a function telling (Turing's Oracle) if the underlying universal Turing machine,
while simulating rule-book $r_k$ will or will not halt on input string `$s$'.
So, if $h(s,r_k) = 1$, the Turing machine halts, if $h(s,r_k) = 0$, it does not halt.
 
Let's define $C_h(k)$ as :-
$$
C_h(k) = \{ s \in C(r_k) \; : \;  \forall l \in R \; \; h(s,r_l) = 1 \}   
$$

Then, the language class accepted by the sequential learner is precisely :-
$$
C_S = \bigcup_{k} C_h(k)
$$

It is obvious that in general $C_h(k) \subseteq C(r_k)$ , 
and therefore, in general, when there exists at least a string `$s_k$' in rule table $r_k$ which 
makes the underlying Universal Turing machine simulating on rule table $r_j$
get into infinite loop ($h(s_k,r_j) = 0$), we would have:-

$$
C_S \subset \bigcup_{k} C(r_k)
$$
That completes the proof.
\end{proof}

It can be argued however that by \emph{cleverly} ordering the selection 
of the rule tables one can possibly complete (definition \ref{complete-system}) the learner.
But in general that will be impossible. 
Next theorem establishes this fact.

\begin{theorem}\label{reducibility}
\textbf{Sequential learner can not be completed.}

Algorithmic modification of a sequential learner into completeness is impossible.
\end{theorem}
\begin{proof}
Let $r_k,r_j,r_h \in R$ are rule tables.

We demonstrate using the worst case scenario, which is
$\forall r_k \exists w_k$ such that simulating $r_j$ with input $w_k$
gets the universal turing machine into infinite loop.

Obviously,  it is impossible to reduce such a system into a complete system.
Now, assume that there is only one $w_k \in r_k$ such that 
simulating $r_h \in R$ with input $w_k$ 
gets the universal turing machine into infinite loop.

But that is unknown unless we are looking at the system 
from outside of the system. We must then already have a table which shows
what strings from which rule table class produce a infinite loop in which 
rule table.

The question is can we create such a table algorithmically?
The answer is no, because we would never know which string 
would gets the simulation into infinite loop.
The function $h(s,r_k)$ is not computable, by Turing machine halting theorem \cite{at1}.
Therefore, automatic creation of such a table is not possible.

Hence, automatic completion of the sequential system, is 
not possible either, as was stated.
\end{proof}

Therefore, it is now established by the Theorem \ref{non-union} that a sequential learner
does not have have closure of the language class for which it poses
the rule tables never the less.
Also, it is not possible to automatically complete it, as stated by Theorem \ref{reducibility}.  

Although the rules in the rule table lets an universal Turing machine precisely 
accept the language class $C_x$ , but the learner might still not accept all the strings 
from the class. This is obviously not efficient.

The next learner, does not have that inefficiency, but it achieves that
goal with extra processing units, and space.
The next learner, of course is the parallel learner.
It utilizes the concept of parallel running Turing Machines \cite{jw}
having dedicated tapes each, but communicate with a monitoring Turing Machine
using a shared tape.

\begin{definition}\label{parallel}
\textbf{ Parallel Learner. }

Let  $n = |R|$.
Then the parallel learner has $n+1$ Universal Turing Machines.
$n$ of them are having single dedicated tape, and a shared tape 
which everyone shares.

To accept an input string $w \in \Sigma^+ $ 
$w$ is  to be supplied to the shared tape, 
from where all the Turing Machines copy
the string into their own dedicated tape, 
and start running the simulation. 
If one of them $\mathcal{TM}_a$ could accept the string,
$\mathcal{TM}_a$ writes back a symbol $\mathcal{V} \not \in \Sigma $ 
to the leftmost cell of the shared tape.

The last Turing Machine only reads the left most cell of the shared tape for 
the symbol $\mathcal{V}$ in a loop. If it has found the symbol,
it accepts the string, as the system has accepted it and halts.

\end{definition}

\begin{theorem}\label{union}
\textbf{Language Class Accepted by the Parallel Learner.}

Parallel learner accepts a language class $C_P$ which is strictly 
the union of the language class of the all the rule tables $r_k$s.
$$
C_P = \bigcup_{k} C(r_k)
$$
\end{theorem}
\begin{proof}
This is elementary.
The problem of a simulated system getting into an infinite loop
is solved by the $n$ Turing Machine running parallel without affecting
each others progress. Also the last Turing machine monitoring 
any of the simulation output ensures if any Turing machine halts with accept,
the learner would halt.   
Therefore, the parallel system is strictly complete over the rules.
\end{proof}
\end{section}

\begin{section}{Comparisons of Learning Strategies}\label{comparison}
In the previous section, we have established two
learning strategies, one in tandem, or the sequential strategy (definition \ref{tandem}),
another - the parallel strategy (definition \ref{parallel}).
In the current section we establish the pros and cons of the different strategies.

Note that we are discussing a \emph{blind} design, that is no conscious optimization, 
ever. None is looking at the system from the outside, and improving the design 
or the wiring of the system.

\begin{theorem}\label{completeness}
\textbf{Completeness of the Learning Models.}

Parallel Learners are complete (definition \ref{complete-system}) while
Sequential Learners are not.
\end{theorem}
\begin{proof}
The proof is a direct consequences of theorem \ref{union} 
and theorem \ref{non-union}. 
\end{proof}

By this theorem, one clear thing we have established is that the  parallel system 
is more \emph{optimal} than the sequential system, in general.
That is, a parallel system can use all it's resources to gain maximum coverage
on the strings those are to be accepted, while a sequential learner can not.
Also by Theorem \ref{reducibility} we have established that there is no autonomous
way to complete the sequential system.
So, parallel systems has inherent evolutionary advantage.

But, this is not the only metric in which the learners to be measured
for optimality. In computer science, the time and space complexity are of utmost importance.

Formally the time complexity question is \emph{how much time it takes to accept a string?}

It is not too hard to show that the parallel strategy wins here.
We show it in the next theorem:-
\begin{theorem}\label{t-complexity}
\textbf{Time Complexity Comparison of the Learners.}

If the time taken to accept a string $w$ in sequential learner is $t_s(w)$
and in parallel learner is $t_p(w)$, then 
$$
t_p(w) \le t_s(w) \; \forall w 
$$
given both using same rule tables, and same Universal Turing Machines.
\end{theorem}
\begin{proof}
This is pretty elementary.
For sequential learner acts in tandem, the time taken to accept 
a string is precisely the time taken to reject the string 
by all the other simulation previously plus the time taken to accept it.
That gives:-
$$
t_s(w) = \sum_{j=1}^{k-1} t_r(w,j) + t_a(w,k)
$$
where the string gets accepted using the $k$'th rule table,
and $t_r(w,j)$ are time taken to reject the string by $j$th rule table,
with $t_a(w,k)$ is time taken to accept the string by simulating $k$th rule table.

However, for the parallel case :-
$$
t_p(w) = t_a(w,k)
$$
Therefore, 
$$
t_s(w) = \sum_{j=1}^{k-1} t_r(w,j) + t_p(w) 
$$
which clearly establishes the theorem.
\end{proof}

We should actually ask the storage complexity of the simulations.
That is, to set up the sequential learner, and the parallel learner,
how much storage space is needed?

We note that the storage space $\sigma$ in either case is a function of the number
of rule tables $n$, because one needs to incorporate the newly added rule table.  
Given the storage required to store the rule table $R_k$ be $\sigma_r(k)$,
and an Universal Turing Machine $\sigma_t$ we have the precise relation
as the next theorem:-
  
\begin{theorem}\label{st-space}
\textbf{Storage Space Comparison of the Learners.}

If storage space of the sequential learner is $\sigma_s$
and in parallel learner is $\sigma_p$, then 
$$
\sigma_s = \sigma_p 
$$
given both using same rule tables, and same Universal Turing Machines.
\end{theorem}
\begin{proof}
We note that:-
$$
\sigma_s = \sigma_t + \sum_{k=1}^n \sigma_r(k) 
$$ 
The parallel machine needs to have $n+1$ Turing machines, 
but the rule for every universal turing machine is the same.
And that is never going to get changed. Then only copy 
of the rule table for the Universal Turing machine itself suffices.
Then
$$
s_p = \sigma_t + \sum_{k=1}^n \sigma_r(k)  
$$ 
which would imply that $\sigma_s = \sigma_p$. 
\end{proof}

Now we ask the space complexity of the simulations done
both by the sequential and the parallel learner.

\begin{theorem}\label{s-complexity}
\textbf{Space Complexity Comparison of the Learners.}

If the space used to accept a string $w$ in sequential learner is $s_s(w)$
and in parallel learner is $s_p(w)$, then 
$$
s_s(w) \le s_p(w) \; \forall w 
$$
given both using same rule tables, and same Universal Turing Machines.

\end{theorem}
\begin{proof}
The proof is again elementary,
we establish the precise relation between them.
Assume that the space required to reject the string by $j$th simulation 
be $s_r(w,j)$ and to accept by $k$th simulation is $s_a(w,k)$.
Then,
$$
s_s(w) = \text{sup} \{ s_r(w,1),... ,s_r(w,k-1) , s_a(w,k) \} 
$$
But, 
$$
s_p(w) =  s_a(w,k) + \sum_{j=1 \; \; j\ne k }^n s_r(w,j)   
$$
which immediately establishes the theorem.
\end{proof}

These theorems clearly establishes that when space is not premium,
then, parallel learner would be a more optimal strategy for learning.
In specificity, if accuracy and completeness is needed, then parallel learners
are better then the sequential ones.

But that is not all. The crux of this lies in the fact that 
autonomous learning in nature has to be inherently blind,
triggered by chance mutations. Given that, there would be no 
way to ensure an \emph{out of the system decision making} to further
optimize the design of the resulting system. 

Due to the nature of the blind learning, even if 
optimization is possible, there has to be parallelism
to complete the learner.
This \emph{intelligent} ordering can evolve in nature,
and is discussed in the next theorem.

\begin{theorem}\label{hybrid-model}
\textbf{Existence of a Hybrid Learning System}

There exists a hybrid model of learning which uses parallelism
and serial model, and is complete.
\end{theorem}
\begin{proof}
We prove it by constructing a hybrid system.
We note that the halting problem establishes a partial order relation
over the rule tables. It can be stated as :-

$r_i \le r_j $ iff $w_j \in C(r_j)$ hangs
the simulation of $r_i$.    

Given this partial order relation we can create 
sets $\mathcal{U}_k$ (for \emph{unrelated} ) so that :-
$$
r_a \in \mathcal{U}_k \text{ iff } \not \exists r_b \in \mathcal{U}_k \text{ s.t. } r_a \le r_b \text{ or } r_b \le r_a  
$$  
Let $ U = \{ \mathcal{U}_k \}$. Now individual rules $r_{ki} \in \mathcal{U}_k$
can be run in sequential. But cross set rules $r_{ki} \in \mathcal{U}_k$ 
and $r_{pi} \in \mathcal{U}_p$ can not be run in sequential.

So, if a system is capable of running the $n = |U|$ parallel execution,
we can make a complete hybrid system that runs \emph{related} rules parallel,
but unrelated rules can be run sequentially.

This completes the proof.  
\end{proof}

We  note that construction of such a system is not algorithmic,
that is not computationally possible.
The step of finding out the relations between language class rules,
(in specificity $h(s,r_k)$) is not computable. 
Hence, only blind mutation can create such a system, that is by chance.
However, up until this section we have established that parallelism, 
in general,  is inherent in nature, even if a Hybrid system 
is evolved, there is parallelism inherent in it.
In the next section we show that
why parallel learning strategy dominates nature. 

\end{section}

\begin{section}{Parallel Strategies and Evolution}\label{evo-l}

In this section we finally ask the following question :

\emph{Given both the sequential and parallel strategies available in Nature,
which strategy would eventually dominate?}

This obviously should depend upon the concept of which strategy
\emph{pays off more} for survival. But payoffs like this 
are in the realm of \emph{Game Theory} \cite{ben}.
In the context of the Game Theory payoffs are represented as numbers which 
represent the motivations of players. Payoffs may represent profit, quantity, that is any ``utility''.

As resources are limited in nature, gaining advantage
or loosing advantage can be modelled by a \emph{fixed sum game}\cite{ben} where
a both the players are competing for a fixed sum in reward.
However, we can generalize any \emph{fixed sum game} into a
\emph{zero sum game} by setting the fixed amount at $0$ \cite{ben}\cite{rd1}\cite{sp}.

Given \emph{accepting all strings having the same utility}, 
we can have the payoffs \emph{proportional to the cardinality of the language class 
accepted by the strategy}.
In general, if learner $p_1$ plays with $S_1$ learning strategy,
and learner $p_2$ plays with $S_2$ learning strategy, then, intuitively, 
the player $p_1$ is in advantage if $|C(S_1)| > |C(S_2)|$,
and for $p_2$ it is vice versa with $|X|$ denoting the cardinality of 
set $X$.
Hence, payoff of $S_1$ against $S_2$ denoted by $E(S_1,S_2)$ can be calculated as:-
$$
E(S_1,S_2) = |C(S_1)| - |C(S_2)|  
$$
When $|C(S)|$ becomes infinite, to define the payoff formally 
we have to resort to the \emph{measure theory} \cite{jd}.

\begin{definition}\label{utility-measure}
\textbf{Utility Measure.}

Let $A,B,X,Y \subset \Sigma^+$ be a set of strings. 
A measure $\mathcal{U} : \Sigma^+ \to \mathbb{R}_+$ 
is a utility measure, iff:-
$$
\mathcal{U}(X) = 0 \text{ iff } X = \varnothing 
$$
implying:-
$$
A \subset B \implies \mathcal{U}(B) > \mathcal{U}(A) 
$$
where $\mathcal{U}(X) > \mathcal{U}(Y) $ 
implies that the set $X$ has more utility than $Y$.
\end{definition}
Note that with this measure in place, we do not need any assumptions about the
utilities of different strings. It also glorifies the age old saying: 
\emph{``no learning is useless''}, only this time, more formally.

\begin{definition}\label{payoff-l}
\textbf{Payoff between Learning Strategies.}

Let $S_1,S_2 \in \mathcal{S}$ are two learning strategies.
Let $C(S_i)$ denotes the language class accepted 
by the strategy $S_i$. Let $\bigcup C(S_k) = C(\mathcal{S})$. 
Let $\mathcal{U}$ be a utility measure \eqref{utility-measure}
defined over $C(\mathcal{S})$.
Then, the payoff of strategy $S_1$ against  $S_2$ is given by:-
$$
E(S_1,S_2) = \mathcal{U}( C(S_1) ) - \mathcal{U}( C(S_2) )  
$$
In particular, if $S_1 = S_2 = S$ then,
$$
E(S,S) = 0 
$$ 
\end{definition}

Now we make the bold claim that \emph{parallel strategy is evolutionarily stable}.
To do so we need to state the definition of 
an \emph{evolutionarily stable strategy} \cite{rd2}, \cite{rd1},\cite{ben} \cite{sp}.

\begin{definition}\label{ess}
\textbf{Evolutionarily stable strategy(ESS).}
Let $\mathcal{S}$ is a set of strategies.
Let $E(S,T)$ represent the payoff for playing strategy $S$ against strategy $T$.
The strategy $S$ is ESS iff \textbf{one of the following conditions 
holds} :- \\ $\forall T \ne S$  with $S,T \in \mathcal{S}$ 
\begin{enumerate}
\item{ Strict Nash Equilibrium : 
$$
\forall T \in \mathcal{S} \; ; \; E(S,S) > E(T,S)
$$
}
\item{ Maynard Smith's second Condition: 
$$
\forall T \in \mathcal{S} \; ; \; E(S,S) = E(T,S) \text{ and } E(S,T) > E(T,T) 
$$
}
\end{enumerate}
\end{definition}

Now with this definition in hand, we establish 
the criterion for ESS in the evolution of learning.
\begin{theorem}\label{ess-l}
\textbf{ESS for Learning Strategies.}

Let $\mathcal{S}$ be a set of learning strategies.
Let $S_e \in \mathcal{S}$ be a learning strategy with:-
$$
\forall S_k \in \mathcal{S} \; ; \; C(S_k) \subset C(S_e)
$$ 
where $C(S)$ is language class accepted by strategy $S$. Then, the strategy $S_e$ is ESS.
\end{theorem}
\begin{proof}
We know that $\exists C_E(k)$ such that:- 
$$
C(S_e) = C(S_k) \cup C_E(k) \; ; \; C_E(k) \cap C(S_k) = \varnothing
$$
because $C(S_k) \subset C(S_e) $ , a proper subset.
As measure is additive, the above relation implies:-
$$
\mathcal{U}( C(S_e)) = \mathcal{U}( C(S_k)) + \mathcal{U}( C_E(k) ) 
$$

Hence, we note that :- 
$$
E(S_e,S_k) = \mathcal{U}( C(S_e) ) - \mathcal{U}( C(S_k) ) =   \mathcal{U}( C_E(k) )
$$
This implies:-
$$
\forall S_k \in \mathcal{S} \; ; \; E(S_e,S_k) = \mathcal{U}(C_E(k)) \implies E(S_e,S_k) > 0 
$$ 
implying :-
$$
\forall S_k \in \mathcal{S} \; ; \; E(S_k,S_e) = -\mathcal{U}(C_E(k)) \implies  E(S_k,S_e) < 0 
$$ 
By definition we have $E(S_e,S_e) = 0 $
and therefore, 
$$
\forall S_k \in \mathcal{S} \; ; \; E(S_e,S_e) > E(S_k,S_e) 
$$
Comparing from definition (\ref{ess}, rule 1), 
$S_e$ is an ESS. In fact we note that this ESS is a
\emph{strict Nash Equilibrium.}
\end{proof}

Now we establish that the parallel strategies are ESS.
 
\begin{theorem}\label{ess-p}
\textbf{Parallel learning strategies are ESS.}

Let $\mathcal{S}$ be a set of strategy such that $\forall S \in \mathcal{S}$ 
has rule table sets $R(S)$ such that no complete tandem learner is possible.
Let $S_p \in \mathcal{S}$ is a parallel strategy. $S_p$ is an ESS.
\end{theorem}
\begin{proof}
We note that the language class accepted by the incomplete 
tandem learners are a strict subset of the language class
accepted by the parallel learner by the theorems 
(\ref{non-union},\ref{union},\ref{completeness}).
That is, if $C_t$ is language class for the sequential learners,
and $C_p$ is the language class of parallel learners, then 
$$
C_t \subset C_p
$$
Now using the theorem \ref{ess-l} we can immediately 
deduce that parallel strategies are ESS.  
\end{proof}
\end{section}

\begin{section}{Summary}\label{summary}

In this paper we have demonstrated why in nature parallel strategies
are the optimally suited one. We established this fact using Church Turing Thesis,
and an utility measure  which tallies with common sense.
This demonstrates the power of computer science as a proper science, 
fully capable of describing natural phenomenon, outside the realm of
rather artificial settings of man made computers. 
The bottleneck in nature starts with the impossibility of algorithmically completing a tandem learner. 
Then, if every rule-book $r_j$ has strings
$s_{jk}$ such that it makes simmulation of the $r_k$ into an infinite loop, then ordering 
won't solve the problem of completion. Parallel computation is the only way out then.
This is the reason why parallel strategies once evolved, would dominate nature.

\end{section}

\end{document}